\documentclass[runningheads,a4paper,11pt]{llncs}

\usepackage[parfill]{parskip}
\usepackage{amsmath,amssymb,bbm}
\usepackage{mathtools}
\usepackage{cases}

\usepackage[english]{babel}
\usepackage{dsfont}

\usepackage{microtype}
\usepackage{comment}
\usepackage{subfigure}
\usepackage{algorithm,algorithmic}
\usepackage{color}
\usepackage{appendix}

\usepackage{authblk}

\usepackage{url}
\usepackage[numbers]{natbib}
\usepackage[colorlinks,citecolor=blue,urlcolor=blue,linkcolor=blue,linktocpage=true]{hyperref}
\usepackage{cleveref}
\crefformat{equation}{(#2#1#3)}
\crefrangeformat{equation}{(#3#1#4) to~(#5#2#6)}
\crefname{equation}{}{}
\Crefname{equation}{}{}

\crefname{definition}{\textbf{definition}}{definitions}
\Crefname{definition}{Definition}{Definitions}
\crefname{assumption}{\textbf{assumption}}{assumptions}
\Crefname{assumption}{Assumption}{Assumptions}

\definecolor{maroon}{RGB}{192,80,77}

\newtheorem{thm}{Theorem}
\newtheorem{lem}[thm]{Lemma}

\newtheorem{cor}[thm]{Corollary}
\newtheorem{defi}[thm]{Definition}

\newcommand{\argmin}{\mathop{\mathrm{argmin}}}
\newcommand{\argmax}{\mathop{\mathrm{argmax}}}

\def\E{\mathbb{E}}
\def\P{\mathbb{P}}

\def\R{\mathbb{R}}
\def\cA{\mathcal{A}}
\def\cB{\mathcal{B}}
\def\cD{\mathcal{D}}

\def\cH{\mathcal{H}}
\def\cL{\mathcal{L}}

\def\cX{\mathcal{X}}
\def\cY{\mathcal{Y}}
\def\cZ{\mathcal{Z}}

\title{On-Average KL-Privacy and its equivalence to Generalization for Max-Entropy Mechanisms}
\author{Yu-Xiang Wang$^{1,2}$ \and Jing Lei$^{2}$ \and Stephen E. Fienberg$^{1,2}$}
\institute{$^1$Machine Learning Department, $^2$Department of Statistics,  and the Heinz College,\\
	Carnegie Mellon University, Pittsburgh, PA 15213, USA}

\authorrunning{Yu-Xiang Wang, Jing Lei and Stephen E. Fienberg}
\titlerunning{On-Average KL-Privacy and its equivalence to Generalization}

\begin{document}

\maketitle
\begin{abstract}
	We define On-Average KL-Privacy and present its properties and connections to differential privacy, generalization and information-theoretic quantities including max-information and mutual information. The new definition significantly weakens differential privacy, while preserving its minimalistic design features such as composition over small group and multiple queries as well as closeness to post-processing. Moreover, we show that On-Average KL-Privacy is \emph{equivalent} to generalization for a large class of commonly-used tools in statistics and machine learning that samples from Gibbs distributions---a class of distributions that arises naturally from the maximum entropy principle. In addition, a byproduct of our analysis yields a lower bound for generalization error in terms of mutual information which reveals an interesting interplay with known upper bounds that use the same quantity.
	\keywords{differential privacy, generalization, stability, information theory,  maximum entropy, statistical learning theory, }
\end{abstract}

%\begin{abstract}
%	We define On-Average KL-Privacy and present its properties and connections to differential privacy, generalization and information-theoretic quantities including max-information and mutual information. Specifically, we show that On-Average KL-Privacy is \emph{equivalent} to generalization for a large class of commonly-used tools in statistics and machine learning that samples from Gibbs distributions --- a class of distributions that arises naturally from maximum entropy principle. On-Average KL-Privacy significantly weakens differential privacy, while preserving its minimalistic design features such as composition over small group and multiple queries as well as closeness to post-processing. In addition, a byproduct of our analysis yields a lower bound for generalization error in terms of mutual information which reveals an interesting interplay with known upper bounds in \citet{russo2015controlling} that use the same quantity.
%\end{abstract}

\vspace{-.5cm}
\section{Introduction}
\vspace{-.5cm}

Increasing privacy concerns have become a major obstacle for collecting, analyzing and sharing data, as well as communicating results of a data analysis in sensitive domains. For example, the second Netflix Prize competition was canceled in response to a lawsuit and Federal Trade Commission privacy concerns, and the National Institute of Health decided in August 2008 to remove aggregate Genome-Wide Association Studies (GWAS) data from the public web site, after learning about a potential privacy risk. These concerns are well-grounded in the context of the Big-Data era as stories about privacy breaches from improperly-handled data set appear very regularly (e.g., medical records~\citep{privacybreach1}, Netflix~\citep{narayanan2008robust},  NYC Taxi~\citep{privacybreach2}). These incidences highlight the need for formal methods that provably protects the privacy of individual-level data points while allowing similar database level of utility comparing to the non-private counterpart. 

There is a long history of attempts to address these problems and the risk-utility tradeoff in statistical agencies \citep{duncanfienberg,hundepool,duncanetal} but most of the methods developed do not provide clear and quantifiable privacy guarantees.   Differential privacy \citep{DworkMNS06,dwork2006differential} succeeds in the first task.  While  it allows a clear  quantification of the privacy loss, it provides a worst-case guarantee and in practice it often requires adding noise with a very large magnitude (if finite at all), hence resulting in unsatisfactory utility, cf., \citep{uhler,yu2014scalable,yang}.

%Differential privacy \citep{DworkMNS06,dwork2006differential} succeeds in the first task and it allows a clear mathematical quantification of the privacy loss. However, it is a worst-case guarantee and in practice it often requires adding noise with a very large magnitude (if finite at all), hence resulting in unsatisfactory utility.

A growing literature focuses on weakening the notion of differential privacy to make it applicable  and for a more favorable privacy-utility trade-off. Popular attempts include $(\epsilon,\delta)$-approximate differential privacy \citep{dwork2006our}, personalized differential privacy \citep{ebadi2015personal,liu2015fast}, random differential privacy \citep{hall2011random} and so on. They each have pros and cons and are useful in their specific contexts.  There is a related literature addressing the folklore observation that ``differential privacy implies generalization'' \citep{dwork2014preserving,hardt2014preventing,steinke2014interactive,dwork2015generalization,bassily2015algorithmic,wang2015learning}. 

The implication of generalization is a minimal property that we feel any notion of privacy should have. This brings us to the natural question:
\begin{itemize}
	\item	Is there a weak notion of privacy that is equivalent to generalization?
\end{itemize}
%{\it }   
In this paper, we provide a partial answer to this question. Specifically, we define On-Average Kullback-Leibler(KL)-Privacy and show that it characterizes On-Average Generalization\footnote{We will formally define these quantities.} for algorithms that draw sample from an important class of maximum entropy/Gibbs distributions, i.e., distributions with probability/density proportional to
\begin{equation*}
\exp(-\cL(\text{Output},\text{Data}))\pi(\text{Output})
\end{equation*}
for a loss function $\cL$ and (possibly improper) prior distribution $\pi$.

 We argue that this is a fundamental class of algorithms that covers a big portion of tools in modern data analysis including Bayesian inference, empirical risk minimization in statistical learning as well as the private releases of database queries through Laplace and Gaussian noise adding. From here onwards, we will refer this class of distributions ``MaxEnt distributions'' and the algorithm that output a sample from a MaxEnt distribution ``posterior sampling''.
\vspace{-.5cm}
\paragraph{Related work:}
This work is closely related to the various notions of algorithmic stability in learning theory \citep{kearns1999algorithmic,bousquet2002stability,mukherjee2006learning,shalev2010learnability}. In fact, we can treat differential privacy  as a very strong notion of stability.  Thus On-average KL-privacy may well be called On-average KL-stability. Stability implies generalization in many different settings but they are often only sufficient conditions. Exceptions include \citep{mukherjee2006learning,shalev2010learnability} who show that notions of stability are also necessary for the consistency of empirical risk minimization and distribution-free learnability of any algorithms. Our specific stability definition, its equivalence to generalization and its properties as a privacy measure has not been studied before. KL-Privacy first appears in \citep{barber2014privacy} and is shown to imply generalization in \citep{bassily2015algorithmic}. On-Average KL-privacy further weakens KL-privacy. A high-level connection can be made to leave-one-out cross validation which is often used as a (slightly biased) empirical measure of generalization, e.g., see \citep{mosteller1968data}.

\vspace{-.5cm}
\section{Symbols and Notation}
\vspace{-.5cm}
We will use the standard statistical learning terminology where $z\in \cZ$ is a data point, $h\in\cH$ is a hypothesis and $\ell: \cZ\times\cH \rightarrow \R$ is the loss function. One can think of the negative loss function as a measure of utility of $h$ on data point $z$. Lastly, $\cA$ is a possibly randomized algorithm that maps a data set $Z\in\cZ^n$ to some hypothesis $h\in\cH$. For example, if $\cA$ is the empirical risk minimization (ERM), then $\cA$ chooses $h^*= \argmin_{h\in \cH}\sum_{i=1}^n  \ell(z_i,h)$. 

Just to point out that many data analysis tasks can be casted in this form, e.g., in linear regression, $z_i=(x_i,y_i)\in \R^d\times\R$, $h$ is the coefficient vector and $\ell$ is just $\|y_i-x_i^Th\|^2$; in k-means clustering, $z\in \R^d$ is just the feature vector, $h=\{h_1,...,h_k\}$ is the collection of $k$-cluster centers and $\ell(z,h)=\min_{j}\|z - h_j\|^2$. Simple calculations of statistical quantities can often be represented in this form too, e.g., calculating the mean is equivalent to linear regression with identity design, and calculating the median is the same as ERM with loss function $|z-h|$.

We also consider cases when the loss function is defined over the whole data set $Z\in \cZ$, in this case the loss function is also evaluated on the whole data set by the structured loss $\mathcal{L}: h\times \cZ\rightarrow \R$. We do not require $Z$ to be drawn from some product distribution, but rather any distribution $D$. Generally speaking, $Z$ could be a string of text, a news article, a sequence of transactions of a credit card user, or rather just the entire data set of $n$ iid samples. We will revisit this generalization with more concrete examples later. However we would like to point out that this is equivalent to the above case when we only have one (much more complicated) data point and the algorithm $\cA$ is applied to only one sample.

\vspace{-1em}
\section{Main Results}\label{sec:mainresults}
\vspace{-1em}
We first describe differential privacy and then it will become very intuitive where KL-privacy and On-Average KL-privacy come from. Roughly speaking, differential privacy requires that for any datasets $Z$ and $Z^\prime$ that differs by only one data point, the algorithm $\cA(Z)$ and $\cA(Z')$ samples output $h$ from two distributions that are very similar to each other.
Define ``Hamming distance''
\begin{equation}\label{eq:Z_distance}
d(Z,Z^\prime):= \# \{i=1,...,n: z_i\neq z_i^{\prime}\}\,.
\end{equation}

\begin{defi}[$\epsilon$-Differential Privacy \citep{dwork2006differential}]\label{def:diff_privacy}
	We call an algorithm $\cA$ $\epsilon$-differentially private (or in short $\epsilon$-DP), if
	$$
	\P(\cA(Z)\in H)\leq \exp(\epsilon)\P(\cA(Z^{\prime})\in H)
	$$
	for $\forall\ Z,\ Z^{\prime}$ obeying $d(Z,Z^{\prime})=1$ and any measurable subset $H\subseteq \cH$.
\end{defi}
More transparently, assuming the range of $\cA$ is the whole space $\cH$, and also assume $\cA(Z)$ defines a density on $\cH$ with respect to a base measure on $\cH$\footnote{These assumptions are only for presentation simplicity.  The notion of On-Average KL-privacy can naturally handle mixture of densities and point masses.}, then
$\epsilon$-Differential Privacy requires
$$
\sup_{Z,Z': d(Z,Z')\leq 1}\sup_{h\in\cH} \log\frac{p_{\cA(Z)}(h)}{p_{\cA(Z')}(h)}  \leq \epsilon.
$$
Replacing the second supremum with an expectation over $h\sim \cA(Z)$ we get the maximum KL-divergence over the output from two adjacent datasets. This is KL-Privacy as defined in \citet{barber2014privacy}, and by replacing both supremums with expectations we get what we call On-Average KL-Privacy.
For $Z\in\mathcal Z^n$ and $z\in\mathcal Z$, denote $[Z_{-1},z]\in\mathcal Z^n$ the data set obtained from replacing the first entry of $Z$ by $z$.
Also recall that the KL-divergence between two distributions $F$ and $G$ is $D_{\rm KL}(F\| G)=\mathbb E_F \frac{dF}{dG}$.
\begin{defi}[On-Average KL-Privacy]
	We say $\cA$ obeys $\epsilon$-On-Average KL-privacy for some distribution $\cD$ if 
	%$$
	%\E_{Z\sim \cD^n,z\sim \cD}\mathrm{D}_\mathrm{KL}(\cA(Z)\|\cA([Z_{-1},z]))=\E_{Z\sim \cD^n,z\sim \cD} \E_{\cA(Z)} \left[\log \frac{p_{\cA(Z)}(h)}{p_{\cA([Z_{-1},z])}(h)} \right] \leq \epsilon.
	%$$
		$$
		\E_{Z\sim \cD^n,z\sim \cD}\mathrm{D}_\mathrm{KL}(\cA(Z)\|\cA([Z_{-1},z])) \leq \epsilon.
		$$
\end{defi}
Note that by the property of KL-divergence, the On-Average KL-Privacy is always nonnegative and is $0$ if and only if the two distributions are the same almost everywhere. In the above case, it happens when $z=z'$.

Unlike differential privacy that provides a uniform privacy guarantee for any users in $\cZ$, on-average KL-Privacy is a distribution-specific quantity that measures the amount of average privacy loss of an average data point $z\sim\cD$ suffer from running data analysis $\cA$ on an data set $Z$ drawn iid from the same distribution $\cD$. 

We argue that this kind of average privacy protection is practically useful because it is able to adapt to benign distributions and is much less sensitive to outliers. After all, when differential privacy fails to provide a meaningful $\epsilon$ due to peculiar data sets that exist in $\cZ^n$ but rarely appear in practice, we would still be interested to gauge how a randomized algorithm protects a typical user's privacy.

%It is allowed that if there are some cases that an adversary with enough auxiliary information to identify $z$ through a specific output $h$ as long as the probability of such $(z,h)$ pair is small.

Now we define what we mean by \emph{generalization}. Let the empirical risk $\hat{R}(h,Z) = \frac{1}{n}\sum_{i=1}^n\ell(h,z_i)$ and the actual risk be $R(h) =  \E_{z\sim \cD}\ell(h,z)$.
\begin{defi}[On-Average Generalization]
	We say an algorithm $\cA$ has on-average generalization error $\epsilon$ if 
	$\left|\E R(\cA(Z)) - \E \hat{R}(\cA(Z),Z)\right| \leq \epsilon$.
\end{defi}
This is slightly weaker than the standard notion of generalization in machine learning which requires $\E|R(\cA(Z))-\hat{R}(\cA(Z),Z)| \leq \epsilon$. Nevertheless, on-average generalization is sufficient for the purpose of proving consistency for methods that approximately minimizes the empirical risk.
%bounding population risk for empirical risk minimization and for proving consistency. 
%other algorithms that approximately minimizes the empirical risk, i.e., for proving consistency.

% definition

% other properties 

% small group privacy, compositions
% 
\vspace{-1em}
\subsection{The equivalence to generalization}
\vspace{-.5cm}
It turns out that when $\cA$ assumes a special form, that is, sampling from a Gibbs distribution, we can completely characterize generalization of $\cA$ using On-Average KL-Privacy. This class of algorithms include the most general mechanism for differential privacy --- exponential mechanism \citep{mcsherry2007mechanism}, which casts many other noise adding procedures as special cases. We will discuss a more compelling reason why restricting our attention to this class is not limiting in Section~\ref{sec:maxent}.
\begin{thm}[On-Average KL-Privacy $\Leftrightarrow$ Generalization]\label{thm:characterization}
	Let the loss function $\ell(z,h) = -\log p(z|h)$ for some model $p$ parameterized by $h$, and let 
	$$\cA(Z):  h \sim p(h|Z) \propto \exp\left(-\sum_{i=1}^n\ell(z,h) - r(h)\right).$$
	If in additional $\cA(Z)$ obeys that for every $Z$, the distribution $p(h|Z)$ is well-defined (in that the normalization constant is finite), then $\cA$ satisfy $\epsilon$-On-Average KL-Privacy \emph{if and only if} $\cA$ has on-average generalization error $\epsilon$.
	%the following two statements are equivalent:
	%\begin{enumerate}
%		\item $\cA$ satisfy $\epsilon$-On-Average KL-Privacy.
%		\item $\cA$ has on-average generalization error $\epsilon$.
		%, in other word,
		%$$
		%\E_{z\sim \cD, Z\sim \cD^n} \E_{h\sim \cA(Z)}  \left|\ell(z,h) - \frac{1}{n}\sum_{i=1}^n \ell(z_i,h)\right| \leq \epsilon.
		%$$
%	\end{enumerate}
\end{thm}
The proof, given in the Appendix, uses a ghost sample trick and the fact that the expected normalization constants of the sampling distribution over $Z$ and $Z'$ are the same.
\vspace{-0.5em}
\begin{remark}[Structural Loss]
	Take $n=1$, and loss function be $\cL$. Then for an algorithm $\cA$ that samples with probability proportional to $\exp(-\cL(h,Z) - r(h))$:	
	$\epsilon$-On-Average KL-Privacy is equivalent to $\epsilon$-generalization of the structural loss.
\end{remark}
\vspace{-0.5em}
\begin{remark}[Dispersion parameter $\gamma$]
	The case when $\cA \propto \exp(-\gamma[\cL(h,Z) - r(h)])$ for a constant $\gamma$ can be handled by redefining $\cL' = \gamma\cL$. In that case, $\epsilon_\gamma$-On-Average KL-Privacy with respect to $\cL'$ implies $\epsilon_{\gamma}/\gamma$ generalization with respect to $\cL$. For this reason, larger $\gamma$ may not imply strictly better generalization. %We will provide an example later where the generalization error is a constant over $\gamma$.
\end{remark}
\vspace{-0.5em}
\begin{remark}[Comparing to differential Privacy]
	Note that here we do not require $\ell$ to be uniformly bounded, but if we do, i.e. $\sup_{z\in\cZ,h\in\cH}|\ell(z,h)|\leq B$, then the same algorithm $\cA$ above obeys $4B\gamma$-Differential Privacy~\cite{mcsherry2007mechanism,wang2015} and it implies $O(B\gamma)$-generalization. This, however, could be much larger than the actual generalization error (see our examples in Section~\ref{sec:exp}).
\end{remark}
	
	\vspace{-1em}
\subsection{Preservation of other properties of DP}
\vspace{-0.5em}
We now show that despite being much weaker than DP, On-Average KL-privacy does inherent some of the major properties of differential privacy (under mild additional assumptions in some cases).
\begin{lem}[Closeness to Post-processing]
	Let $f$ be any (possibly randomized) measurable function from $\cH$ to another domain $\cH'$, then for any $Z,Z'$
	$$
	\text{D}_{\mathrm{KL}}(f(\cA(Z))\| f(\cA(Z'))) \leq \text{D}_{\mathrm{KL}}(\cA(Z)\| \cA(Z')). 
	$$
\end{lem}
\begin{proof}
	This directly follows from the data processing inequality for the R{\'e}nyi divergence in \citet[Theorem 1]{van2014renyi}.
	%$\cA$ can be thought of as a conditional distribution operator.
	% $f(\cA(Z))$ is a random variable induced by the distribution of $\cA$ and $f$ so its distribution is a linear completely positive trace-preserving (CPTP) map of the original distribution of $\cA(Z)$.
	%  The claim then follows from the monotonicity of quantum relative entropy \citep{lindblad1975completely}.\qed
\end{proof}

%Note that the proof uses a deep result in quantum physics, and consequently the assumption of discrete domain. It is unclear whether the same inequality extends to relative entropy on the continuous domain. 
%\red{Please check whether we can prove this through some other means.}

\begin{lem}[Small group privacy]\label{lem:smallgroup_privacy}
	Let $k\leq n$. 
	\begin{comment}
	Let $\cA$ obeys
	\begin{equation}\label{eq:monotone_A}
	\E_{h\sim\cA(Z)}\log \frac{p_{\cA(Z')}(h)}{p_{\cA(Z'')}(h)} \leq\E_{h\sim\cA(Z')}\log \frac{p_{\cA(Z')}(h)}{p_{\cA(Z'')}(h)} .
	\end{equation}
	holds for any $h,Z,Z',Z''$ such that $d(Z',Z)\leq d(Z'',Z)$.
	\end{comment}
	Assume $\cA$ is posterior sampling as in Theorem~\ref{thm:characterization}. Then for any $k=1,...,n$, we have
	{\small
		$$\E_{[Z,z_{1:k}]\sim \cD^{n+k}}D_\mathrm{KL}\left(\cA(Z)\|\cA([Z_{-1:k},z_{1:k}])\right)  = k \E_{[Z,z]\sim \cD^{n+1}}D_\mathrm{KL}\left(\cA(Z)\|\cA([Z_{-1},z])\right).$$
	}
	\begin{comment}
	If in addition
	\begin{enumerate}
	\item 	$\cA$ obeys $\epsilon$-KL-Privacy for datasets of size $n$, then
	$$
	\sup_{d(Z,Z')\leq k} \E_{h\sim \cA(Z)} \log p_{\cA(Z)}(h) - p_{\cA(Z')}  \leq k\epsilon.
	$$
	\item  $\cA$ obeys $\epsilon$-On-Average-KL-Privacy for a distribution $\cD$, then
	$$
	\E_{Z\sim \cD^n}\E_{z_1,...,z_k\sim\cD^k} \E_{h\sim \cA(Z)} \log p_{\cA(Z)}(h) - p_{\cA([Z_{-1:k},z_1,...,z_k])}  \leq k\epsilon.
	$$
	\end{enumerate}
	\end{comment}
\end{lem}
\begin{lem}[Adaptive Composition Theorem]\label{lem:composition}
	Let $\cA$ be $\epsilon_1$-(On-Average) KL-Privacy and $\cB(\cdot,h)$ be $\epsilon_2$-(On-Average) KL-Privacy for every $h\in \Omega_\cA$ where the support of random function $\cA$ is $\Omega_\cA$. Then $(\cA,\cB)$ jointly is ($\epsilon_1+\epsilon_2$)-(On-Average) KL-Privacy.
\end{lem}
We prove Lemma~\ref{lem:smallgroup_privacy}~and~\ref{lem:composition} in the appendix.

\vspace{-1em}
\subsection{Posterior Sampling as Max-Entropy solutions}\label{sec:maxent}
\vspace{-0.5em}
In this section, we give a few theoretical justifications why restricting to posterior sampling is not limiting the applicability of Theorem~\ref{thm:characterization} much. First of all, Laplace, Gaussian and Exponential Mechanism in the Differential Privacy literature are special cases of this class. Secondly, among all distributions to sample from, the Gibbs distribution is the variational solution that simultaneously maximizes the conditional entropy and utility. To be more precise on the claim, we first define conditional entropy.
\begin{defi}[Conditional Entropy]
	Conditional entropy
	$$H(\cA(Z)|Z)  = -\E_{Z}\E_{h\sim \cA(Z)} \log p(h|Z)$$ 
	where  $\cA(Z) \sim p(h|Z)$.
\end{defi}
\begin{thm}
	Let $Z\sim \cD^n$ for any distribution $\cD$. A variational solution to the following convex optimization problem
	\begin{equation}
	\begin{aligned}
	\min_{\cA} &\quad -\frac{1}{\gamma} \E_{Z\sim \cD^n}H(\cA(Z)|Z)+ \E_{Z\sim \cD^n} \E_{h\sim \cA(Z)} \sum_{i=1}^n\ell_i(h,z_i)
	\end{aligned}
	\end{equation}
	is $\cA$ that outputs $h$ with distribution $
	p(h|Z) \propto \exp\left( -\gamma\sum_{i=1}^n\ell_i(h,z_i)\right).
	$
\end{thm}
\begin{proof}
	This is an instance of Theorem~3 in \citet{mir2013information} (first appeared in \citet{tishby2000information}) by taking the distortion function to be the empirical risk. Note that this is a simple convex optimization over the functions and the proof involves substituting the solution into the optimality condition with a specific Lagrange multiplier chosen to appropriately adjust the normalization constant. \qed
\end{proof}
The above theorem is distribution-free, and in fact works for every instance of $Z$ separately. Condition on each $Z$, the variational optimization finds the distribution with maximum information entropy among all distributions that satisfies a set of utility constraints. This corresponds to the well-known principle of maximum entropy (MaxEnt)~\citep{jaynes1957information}. Many philosophical justifications of this principle has been proposed, but we would like to focus on the statistical perspective and treat it as a form of regularization that penalizes the complexity of the chosen distribution (akin to Akaike Information Criterion~\citep{akaike1981likelihood}), hence avoiding overfitting to the data.
%This class includes the entire exponential family distributions \citep{barndorff2014information}.
For more information, we refer the readers to references on  MaxEnt's connections to thermal dynamics \citep{jaynes1957information}, to Bayesian inference and convex duality \citep{altun2006unifying} as well as its modern use in modeling natural languages \citep{berger1996maximum}.

%We argue that the maximum entropy solutions are fundamental and it provides variational characterization for the entire exponential family class of distributions.

Note that the above characterization also allows for any form of prior $\pi(h)$ to be assumed. Denote prior entropy $\tilde{H}(h) = -\E_{h\sim\pi(h)}\log(\pi(h))$, and define information gain $\tilde{I}(h;Z) = \tilde{H}(h)- H(h|Z)$. The variational solution of $p(h|Z)$ that minimizes $\tilde{I}(h;Z) + \gamma \E_Z\E_{h|Z} \cL(Z,h)$ is proportional to $\exp(-\cL(Z,h))\pi(h)$. This provides an alternative way of seeing the class of algorithms $\cA$ that we consider. $\tilde{I}(h;Z)$ can be thought of as a information-theoretic quantification of privacy loss, as described in \citep{zhou2009compressed,mir2013information}. As a result, we can think of the class of $\cA$ that samples from MaxEnt distributions as the most private algorithm among all algorithms that achieves a given utility constraint.

%Therefore, arguably whenever there is an $\cA$ that obeys On-Average-KL-Privacy with rate $\epsilon$, there must be another algorithm $\cA'$ that is a posterior sampling mechanism which maximizes the expected utility. This suggests that we can just stick to posterior sampling. Specifically, suppose we define ``utility'' to be the extent that we are able to minimize the empirical risk.

%Interestingly, this lemma also implies that the optimal objective function does not depend on $\epsilon$.	
%$$I(Z;\cA(Z)) + \E_{Z\sim \cD^n} \E_{h\sim \cA(Z)} \sum_{i=1}^n\ell_i(h,z_i)= \E_{Z\sim \cD^n}\int_h \sum_{i=1}^n -\ell(h,z_i) dh$$

%\section{Connections to Information-Theoretic quantities and other DP variants}
%In this section, we show connections to well-studied notions related to privacy and generalization in the literature. 
\vspace{-.5cm}
\section{Max-Information and Mutual Information} 
\vspace{-.5cm}
Recently, \citet{dwork2015generalization} defined approximate max-information %$I_\infty^\beta(A,n)$
 and used it as a tool to prove generalization (with high probability). \citet{russo2015controlling} showed that the weaker mutual information %$I(\cA(Z);\cL(\cdot,Z))$
  implies on-average generalization under a distribution assumption of the entire space of $\{\cL(h,Z)|h\in\cH\}$ induced by distribution of $Z$. In this section, we compare On-Average KL-Privacy with these two notions. Note that we will use $Z$ and $Z'$ to denote two completely different datasets rather than adjacent ones as we had in differential privacy.
%\textcolor{red}{[Jing: need clearer meaning for $\{\cL(h,Z)|h\in\cH\}$. Also need to define $I_\infty^\beta(A,n)$.]}
\begin{defi}[Max-Information, Definition 11 in \citep{dwork2015generalization}]
	We say $\cA$ has an $\beta$-approximate max-information of $k$ if for every distribution $\cD$, 
	$$I^\beta_\infty(Z;\cA(Z)) = \max_{(H,Z)\subset\cH\times\cZ^n: \P(h\in H,Z\in \tilde{Z}) >\beta} \log\frac{\P(h\in H,Z\in \tilde{Z})-\beta}{\P(h\in H)p(Z\in\tilde{Z})} \leq k.$$
	%$I^\beta_\infty(Z;\cA(Z))\leq k$
	 This is alternatively denoted by $I_\infty^\beta(\cA,n)\leq k$.  We say $\cA$ has a pure max-information of $k$ if $\beta=0$.
\end{defi}
It is shown that differential privacy and short description length imply bounds on max-information \citep{dwork2015generalization}, hence generalization. 
Here we show that the pure max-information implies a very strong On-Average-KL-Privacy for any distribution $\cD$ when we take $\cA$ to be a posterior sampling mechanism. 
\begin{lem}[Relationship to max-information]\label{lem:maxinfo}
	If $\cA$ is a posterior sampling mechanism as described in Theorem~\ref{thm:characterization}, then
	$I_\infty(\cA,n) \leq k$ implies that $\cA$ obeys $k/n$-On-Average-KL-Privacy for any data generating distribution $\cD$.
\end{lem}

\begin{comment}
Interestingly, we do not have the same results for KL-Privacy using the above arguments, unless we can construct
$$ (Z,Z')\in \argmax_{Z,Z'\in \cZ^n}\sum_i \text{KL}(\cA([Z_{-1:(i-1)},Z'_{1:(i-1)}]) \| \cA([Z_{-1:i},Z'_{1:i}]))$$ 
such that $(Z,Z')$ also obeys that for each $i$, 
$$\text{KL}(\cA([Z_{-1:(i-1)},Z'_{1:(i-1)}])\|\cA([Z_{-1:i},Z'_{1:i}])) = \sup_{d(A,B)\leq 1} \text{KL}(\cA(A)||\cA(B)).$$
In general, it might not be possible for all $n$ pairs of adjacent datasets to attain the maximum KL-privacy bound, therefore a naive bound only gives us $k$-KL-privacy.
\end{comment}

An immediate corollary of the above connection is that we can now significantly simplify the proof for ``max-information $\Rightarrow$ generalization'' for posterior sampling algorithms.% and ``DP$\Rightarrow$generalization''.
\begin{cor}
	Let $\cA$ be a posterior sampling algorithm. $I_\infty(\cA,n)\leq k$ implies that $\cA$ generalizes with rate $k/n$. %Suppose $\cA$ is $\epsilon$-DP, then it implies that $I_\infty(\cA,n)\leq \epsilon n$ which further implies that $\cA$ generalizes with $\epsilon$.
\end{cor}

%On the other hand, it is clear from the proof of the lemma that pure max information or differential privacy is way too strong for ensuring generalization. 
We now compare to mutual information and draw connections to \citep{russo2015controlling}.

\begin{defi}[Mutual Information] 
	The mutual information 
	$$I(\cA(Z);Z)  = \E_{Z}\E_{h\sim \cA(Z)} \log \frac{p(h,Z)}{p(h)p(Z)}$$ 
	where  $\cA(Z) \sim p(h|Z)$, $p(Z) = \cD^n$ and $p(h)=\int p(h|Z)p(Z)dZ$.% is the marginal distribution.
\end{defi}

\begin{lem}[Relationship to Mutual Information]\label{lem:mutualinfo}
	For any randomized algorithm $\cA$, let $\cA(Z)$ be an RV, and $Z, Z'$ be two datasets of size $n$. We have
	{\small
	$$I(\cA(Z);Z)  = D_{\mathrm{KL}}(\cA(Z)\|\cA(Z')) + \E_{Z}\E_{h\sim\cA(Z)}\left[\E_{Z'}\log p(\cA(Z')) - \log \E_{Z'} p(\cA(Z'))\right],$$
}
	which by Jensen's inequality implies
	$
	I(\cA(Z),Z) \leq \text{D}_{\mathrm{KL}}(\cA(Z)\|\cA(Z')).
	$
\end{lem}

%We argue that asymptotically, $I(\cA(Z);Z)$ and $D_{\mathrm{KL}}(\cA(Z)\|\cA(Z'))$ are exactly the same. A hand-waving argument is that as $n\rightarrow \infty$, $\cA(Z)$ converges in probability to a fixed $h$, by the continuity of measure, $\E_{Z'}\log p(\cA(Z')) - \log \E_{Z'}p(\cA(Z')) \rightarrow_P 0$. Let the generalization error be $\epsilon_n$, by the small group privacy, $D_{\mathrm{KL}}(\cA(Z)\|\cA(Z')) = n\epsilon$, which is at best a constant. Therefore asymptotically the gap diminishes.
A natural observation is that for MaxEnt $\cA$ defined with $\cL$, mutual information lower bounds its generalization error. On the other hand, Proposition 1 in \citet{russo2015controlling} states that under the assumption that $\cL(h,Z)$ is $\sigma^2$-subgaussian for every $h$, then the on-average generalization error is always smaller than
$\sigma\sqrt{2 I(\cA(Z);\cL(\cdot,Z))}.$ Similar results hold for sub-exponential $\cL(h,Z)$ \citep[Proposition~3]{russo2015controlling}.

Note that in their bounds, $I(\cA(Z);\cL(\cdot,Z))$ is the mutual information between the choice of hypothesis $h$ and the loss function for which we are defining generalization on. 
%Both are random variables induced by the randomness of data $Z$. 
By data processing inequality, we have $I(\cA(Z);\cL(\cdot,Z))  \leq I(\cA(Z);Z)$. Further, when $\cA$ is posterior distribution, it only depends on $Z$ through $\cL(\cdot,Z)$, namely $\cL(\cdot,Z)$ is a sufficient statistic for $\cA$. As a result $\cA \perp Z |\cL(\cdot,Z)$. Therefore, we know  $I(\cA(Z);\cL(\cdot,Z)) = I(\cA(Z);Z)$. Combine this observation with Lemma~\ref{lem:mutualinfo} and Theorem~\ref{thm:characterization}, we get the following characterization of generalization through mutual information.
\begin{cor}[Mutual information and generalization]
	Let $\cA$ be an algorithm that samples $\propto \exp\left(-\gamma \cL(h,Z)\right)$, and $\cL(h,Z)-R(h)$ is $\sigma^2$-subgaussian for any $h\in\cH$, then
$$
\frac{1}{\gamma}I(\cA(Z);Z) \leq \left|\E_Z\E_{h\sim \cA(Z)} [\cL(h,Z) - R(h)] \right| \leq \sigma\sqrt{2 I(\cA(Z);Z)}.
$$
If $\cL(h,Z)-R(h)$ is $\sigma^2$-subexponential with parameter $(\sigma,b)$ instead, then we have a weaker upper bound $b I(\cA(Z);Z) + \sigma^2/(2b)$.
\end{cor}
%This suggests a novel result that at least for constant $\gamma$, mutual information also serves as a lower bound of generalization for MaxEnt sampling. 
%provides the first lower bound of generalization in terms of mutual information. 
The corollary implies that for each $\gamma$
%if we choose $\gamma=1$ (we typically do that for Bayesian inference), 
we have an intriguing bound that says $I(\cA;Z) \leq 2\gamma^2\sigma^2$ for any distribution of $Z$, $\cH$ and $\cL$ such that $\cL(\cdot,Z)$ is $\sigma^2$-subgaussian. 
%If it happens that $I(\cA;Z) =\Theta(\gamma^2\sigma^2)$ for some $\gamma$, then we get a pair of matching upper and lower bound of generalization error at $\Theta(\gamma\sigma^2)$.
One interesting case is when $\gamma=1/\sigma$. This gives
$$\sigma I(\cA(Z);Z) \leq \left|\E_Z\E_{h\sim \cA(Z)} [\cL(h,Z) - R(h)] \right|\leq  \sigma\sqrt{2I(\cA(Z);Z)}.$$
The lower bound is therefore sharp up to a multiplicative factor of $\sqrt{I(\cA(Z);Z)}$. 

%The last inequality implies that $\E_Z\E_{h\sim \cA(Z)} \left|\cL(h,Z) - R(h)\right| \leq \sqrt{5}\sigma$, which suggests that the stronger expected generalization matches (up to constant) the minimax lower bound of $\sigma$ for estimating the mean of a subgaussian random variable despite $\cA$ being a selective procedure.}

%\begin{cor}
%	Let $ \cS = \{(Z,\cH,\cL) : \cL(h,Z) \text{ is }\sigma^2\text{-subgaussian }\forall h\in\cH \}$.
%	$$	\sigma \leq \inf_{\text{any estimator }\cA}\sup_{(Z,\cH,\cL)\in\cS}\E_Z\E_{h\sim \cA(Z)} \left|\cL(h,Z) - R(h)\right| \leq \sqrt{5}\sigma$$
%\end{cor}
%Note that when $\gamma\rightarrow \infty$, namely when $\cA$ becomes empirical risk minimization, the lower bound becomes meaningless, but the KL-privacy remains equal to the generalization error.
\vspace{-.5cm}
\section{Connections to Other Attempts to Weaken DP}
\vspace{-.5cm}
We compare and contrast the On-Average KL-Privacy with other notions of privacy that are designed to weaken the original DP. The (certainly incomplete) list includes $(\epsilon,\delta)$-approximate differential privacy (Approx-DP) \citep{dwork2006our}, random differential privacy (Rand-DP) \citep{hall2011random}, Personalized Differential Privacy (Personal-DP) \citep{ebadi2015personal,liu2015fast} and Total-Variation-Privacy (TV-Privacy) \citep{barber2014privacy,bassily2015algorithmic}. Table~\ref{tab:def_privacy} summarizes and compares of these definitions.

\begin{table}[tb]
	\centering
\resizebox{\textwidth}{!}{
	\begin{tabular}{c|c|c|c}
		\hline
		Privacy definition& $Z$ &  $z$ & Distance (pseudo)metric\\\hline
		Pure DP & $\sup_{Z\in\cZ^n}$ & $\sup_{z\in\cZ}$ & $D_{\infty}(P \| Q)$  \\
		Approx-DP & $\sup_{Z\in\cZ^n}$ & $\sup_{z\in\cZ}$ & $D_{\infty}^{\delta}(P \| Q)$  \\
		Personal-DP & $\sup_{Z\in\cZ^n}$ & for each $z$& $D_{\infty}(P \| Q)$ or $D_{\infty}^{\delta}(P \| Q)$  \\
		KL-Privacy & $\sup_{Z\in\cZ^n}$ & $\sup_{z\in\cZ}$ & $D_{\text{KL}}(P \| Q)$\\
		TV-Privacy & $\sup_{Z\in\cZ^n}$ & $\sup_{z\in\cZ}$ & $\|P-Q\|_{TV}$\\
		Rand-Privacy& $1-\delta_1$ any $\cD^n$& $1-\delta_1$ any $\cD$ & $D_{\infty}^{\delta_2}(P \| Q)$ \\
		On-Avg KL-Privacy& \;$\E_{Z\sim \cD^n}$ for each $\cD$ \;& $\E_{Z\sim \cD}$  for each $\cD$  & $D_{\text{KL}}(P\|Q)$ \\\hline
	\end{tabular}
}
\caption{Summary of different privacy definitions.}\label{tab:def_privacy}
\end{table}
\begin{figure}[tb]
	\centering
	\includegraphics[width=0.9\textwidth]{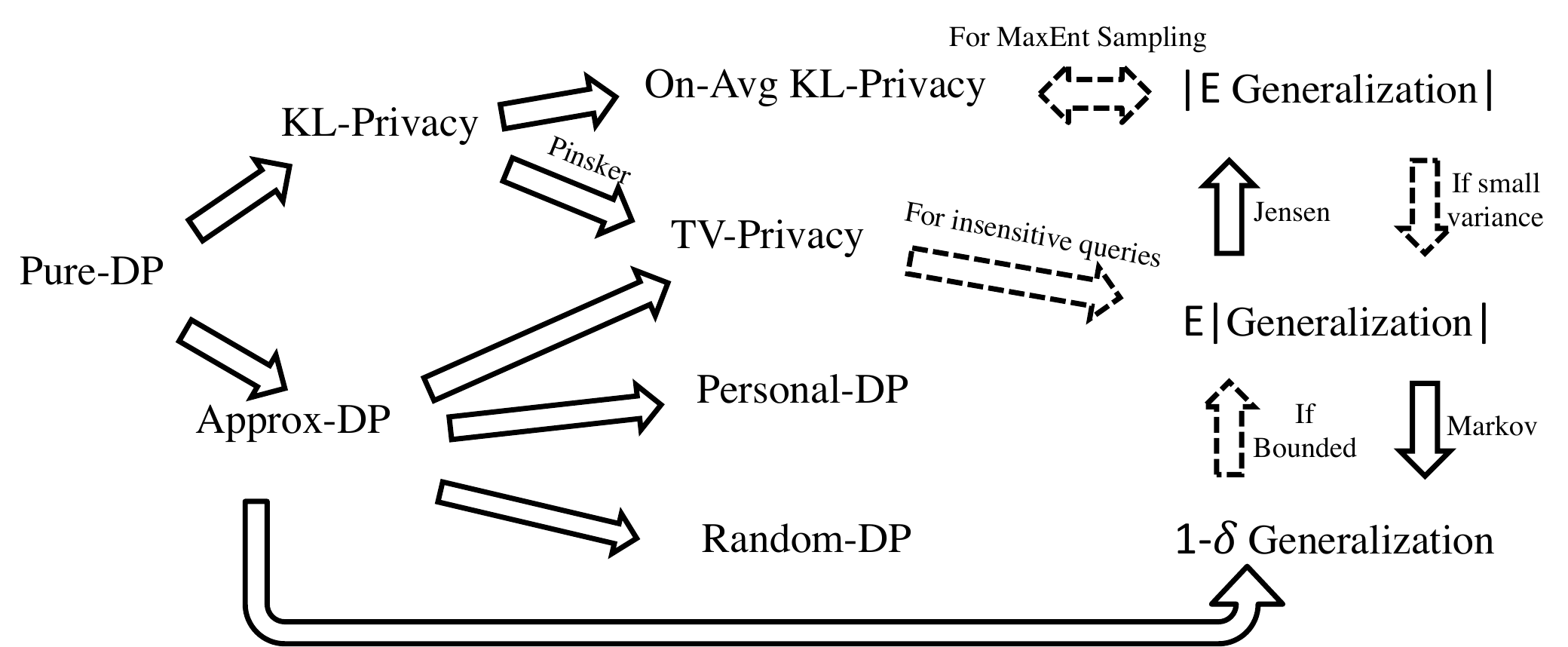}
	\caption{Relationship of different privacy definitions and generalization.}\label{fig:compare}
\end{figure}

A key difference of On-Average KL-Privacy from almost all other previous definitions of privacy, is that the probability is defined only over the random coins of private algorithms. For this reason, even if we convert our bound into the high probability form, the meaning of the small probability $\delta$ would be very different from that in Approx-DP.
The only exception in the list is Rand-DP, which assumes, like we do, the $n+1$ data points in adjacent data sets $Z$ and $Z'$ are draw iid from a distribution. Ours is weaker than Rand-DP in that ours is a distribution-specific quantity.

Among these notions of privacy, Pure-DP and Approx-DP have been shown to imply generalization with high probability \citep{dwork2014preserving,bassily2015algorithmic}; and
TV-privacy was more shown to imply generalization (in expectation) for a restricted class of queries (loss functions) \citep{bassily2015algorithmic}. The relationship between our proposal and these known results are clearly illustrated in Fig.~\ref{fig:compare}. To the best of our knowledge, our result is the first of its kind that crisply characterizes generalization.

Lastly, we would like to point out that while each of these definitions retains some properties of differential privacy, they might not possess all of them simultaneously and satisfactorily. For example, $(\epsilon,\delta)$-approx-DP does not have a satisfactory group privacy guarantee as $\delta$ grows exponentially with the group size.

%To compare with these different definitions, it is useful to return to the beginning of Section~\ref{sec:mainresults} when we describe the two supremums differential privacy take and the idea that it is just a distance measure between two distributions. 

\vspace{-1em}
\section{Experiments}\label{sec:exp}
\vspace{-1em}
In this section, we validate our theoretical results through numerical simulation. Specifically, we use two simple examples to compare the $\epsilon$ of differential privacy, $\epsilon$ of on-average KL-privacy, the generalization error, as well as the utility, measured in terms of the excess population risk.

\begin{figure}[tb]
	\centering
	\includegraphics[width=0.45\textwidth]{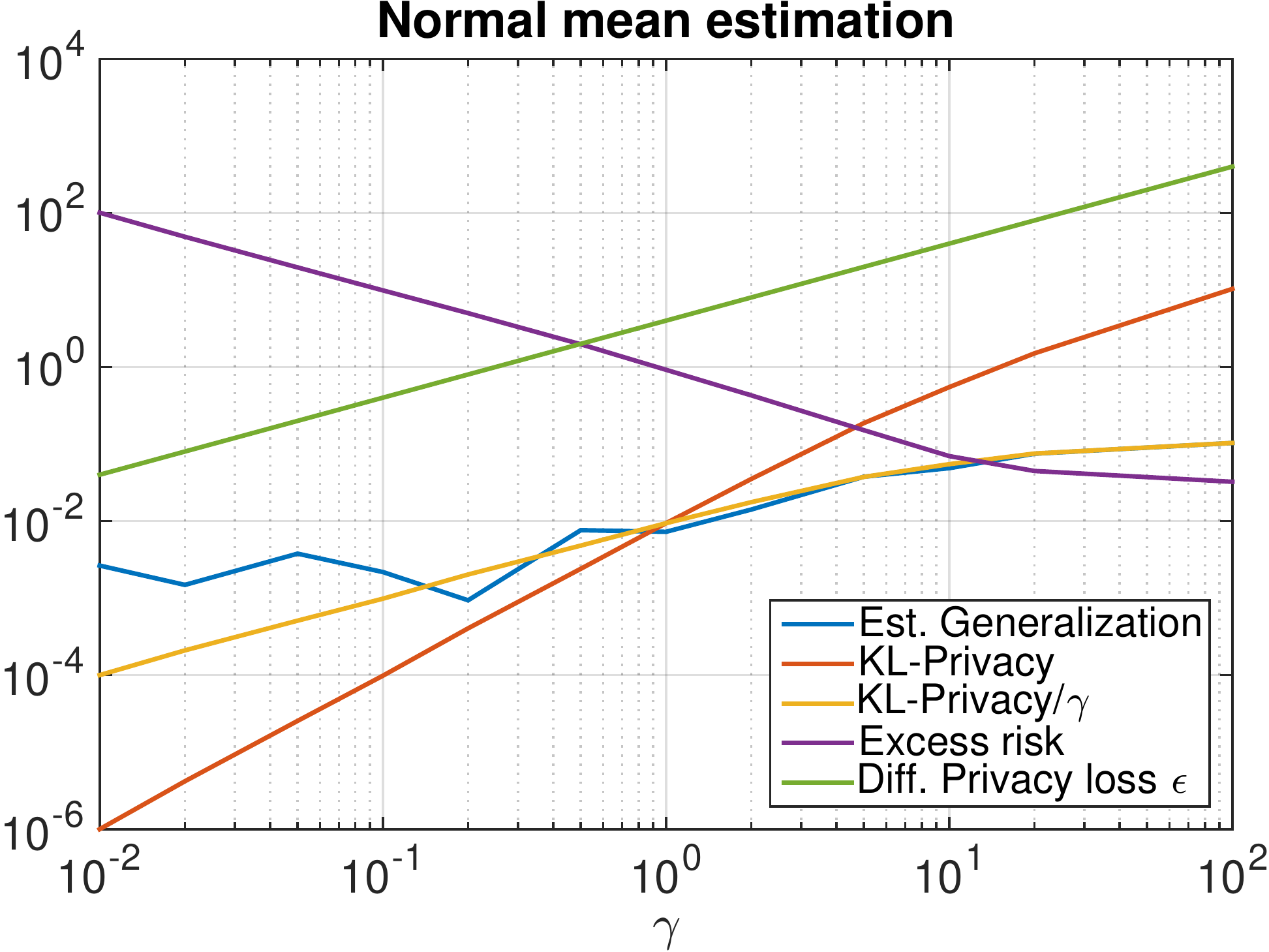}
	\quad
	\includegraphics[width=0.45\textwidth]{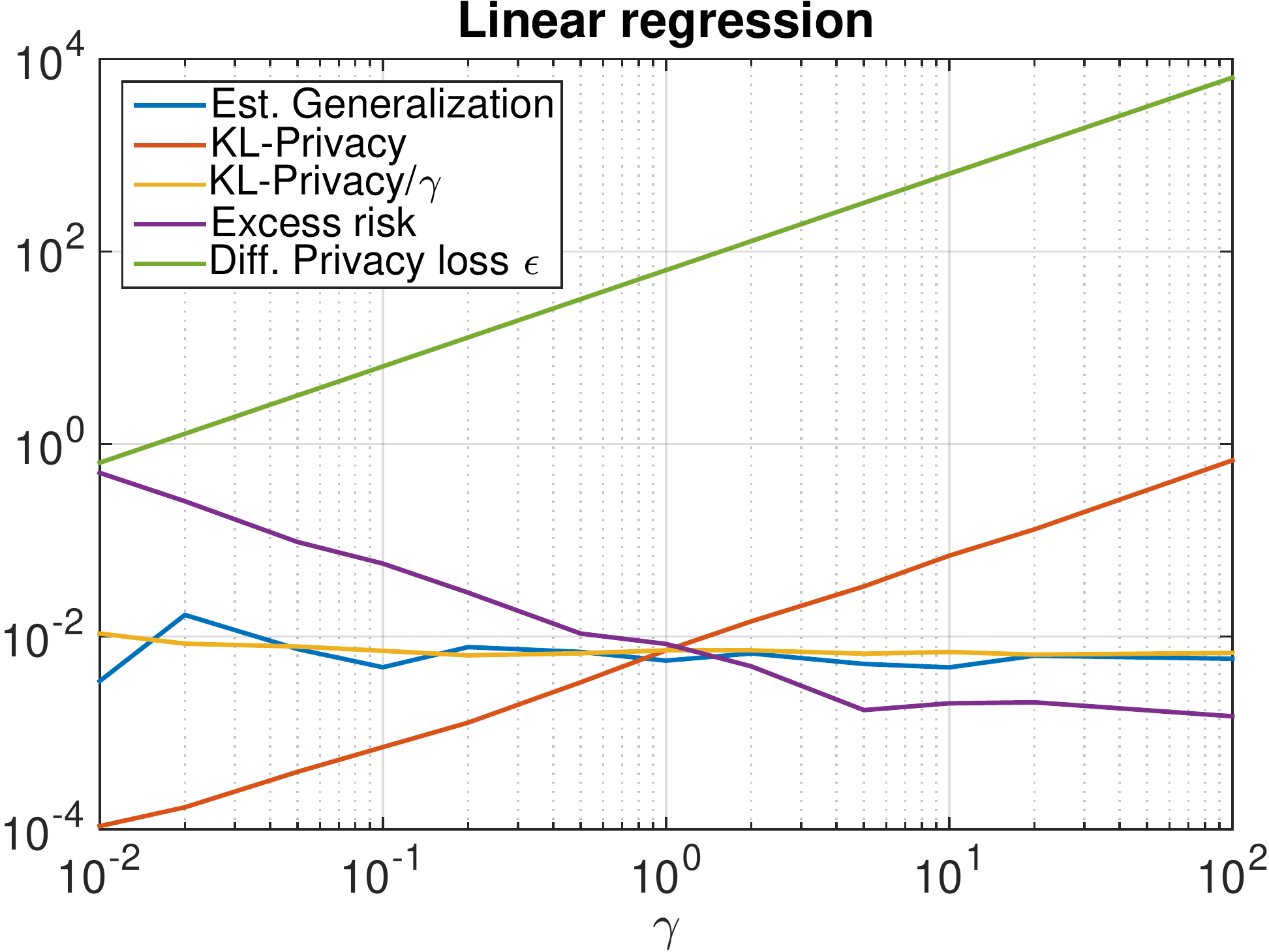}
	\caption{Comparison of On-Avg KL-Privacy and Differential Privacy on two examples.}\label{fig:experiment}
\end{figure}
The first example is the private release of mean, we consider $Z$ to be the mean of $100$ samples from standard normal distribution truncated between $[-2,2]$. Hypothesis space $\cH=\R$, loss function $\cL(Z,h)=|Z-h|$. $\cA$ samples with probability proportional to $\exp(-\gamma|Z-h|)$. Note that this is the simple Laplace mechanism for differential privacy and the global sensitivity is $4$, as a result this algorithm is $4\gamma$-differentially private. 

The second example we consider is a simple linear regression in 1D. We generate the data from a simple univariate linear regression model $y  = xh + \text{noise}$, where $x$ and the noise are both sampled iid from a uniform distribution defined on $[-1,1]$. The true $h$ is chosen to be $1$. Moreover, we use the standard square loss $\ell(z_i,h)= (y_i-x_ih)^2$. Clearly, the data domain $\cZ=\cY\times\cX = [-1,1]\times[-2,2]$ and if we constrain $\cH$ to be within a bounded set $[-2,2]$, $\sup_{x,y,\beta}(y-x\beta)^2 \leq 16$ and the posterior sampling with parameter $\gamma$ obeys $64\gamma$-DP.

Fig.~\ref{fig:experiment} plots the results over an exponential grid of parameter $\gamma$. In these two examples, we calculate on-Average KL-Privacy using known formula of the KL-divergence of Laplace and Gaussian distributions.  Then we stochastically estimate the expectation over data. We estimate the generalization error  in the direct formula by evaluating on fresh samples.
As we can see, appropriately scaled On-Average KL-Privacy characterizes the generalization error precisely as the theory predicts. On the other hand, if we  just compare the privacy losses, the average $\epsilon$ from a random dataset given by On-Avg KL-Privacy is smaller than that for the worst case in DP by orders of magnitudes. 
\vspace{-1em}
\section{Conclusion}
\vspace{-1em}
We  presented On-Average KL-privacy as a new notion of privacy (or stability) on average. We showed that this new definition preserves properties of differential privacy including closedness to post-processing, small group privacy and adaptive composition. Moreover, we showed that On-Average KL-privacy/stability characterizes a weak form of generalization for a large class of sampling distributions that simultaneously maximize entropy and utility. This equivalence and connections to certain information-theoretic quantities allowed us to provide the first lower bound of generalization using mutual information. Lastly, we conduct numerical simulations which confirm our theory and demonstrate the substantially more favorable privacy-utility trade-off.

\newpage
\appendix
\section{Proofs of technical results}
\vspace{-1em}
\begin{proof}[Proof of Theorem~\ref{thm:characterization}]
	We prove this result using a ghost sample trick.
	{\small
		\begin{align*}
		&\E_{z\sim \cD, Z\sim \cD^n}  \E_{h\sim \cA(Z)}\left[ \ell(z,h) - \frac{1}{n}\sum_{i=1}^n \ell(z_i,h)\right]\\
		=&\E_{Z'\sim \cD^n, Z\sim \cD^n}   \E_{h\sim \cA(Z)}\left[ \frac{1}{n}\sum_{i=1}^n\ell(z_i',h) - \frac{1}{n}\sum_{i=1}^n \ell(z_i,h)\right]\\
		=&\frac{1}{n}\sum_{i=1}^n \E_{z_i'\sim \cD, Z\sim \cD^n}   \E_{h\sim \cA(Z)}\left[ \ell(z_i',h) - \ell(z_i,h)\right]\\
		=&\frac{1}{n}\sum_{i=1}^n \E_{z_i'\sim \cD, Z\sim \cD^n}   \E_{h\sim \cA(Z)}\left[ \ell(z_i',h) + \sum_{j\neq i}\ell(z_j,h) +r(h)  - \ell(z_i,h) - \sum_{j\neq i}\ell(z_j,h)  -r(h)\right]\\
		=& \frac{1}{n}\sum_{i=1}^n \E_{z_i'\sim \cD, Z\sim \cD^n}   \E_{\cA(Z)}\left[ - \log p_{\cA([Z_{-i},z'_i])}(h)+\log p_{\cA(Z)(h)}(h)  + \log K_{i}- \log K'_{i}\right]\\
		=& \frac{1}{n}\sum_{i=1}^n \E_{z_i'\sim \cD, Z\sim \cD^n}   \E_{\cA(Z)}\left[ \log p_{\cA(Z)}(h) - \log p_{\cA([Z_{-i},z'_i])}(h) \right]\\
		=&\E_{z\sim \cD, Z\sim \cD^n}   \E_{h\sim \cA(Z)}\left[ \log p_{\cA(Z)}(h) - \log p_{ \cA([Z_{-1},z])}(h)  \right]\,. %\leq  \epsilon
		\end{align*}
	}
	The  $K_i$ and $K_i'$ are partition functions of $p_{\cA(Z)}(h)$ and $p_{\cA([Z_{-i},z_i'])}(h)$ respectively. Since $z_i\sim z_i'$, we know $ \E K_i - \E K_i' = 0.$ The proof is complete by noting that 
	the On-Average KL-privacy is always non-negative and so is the difference of the actual risk and expected empirical risk (therefore we can take absolute value without changing the equivalence).
	\qed
\end{proof}
\begin{proof}[Proof of Lemma~\ref{lem:smallgroup_privacy}]
	Let $k=2$, we have
	\begin{align*}
	&\E_{[Z,z_1',z_2']\sim \cD^{n+2}}\E_{h\sim \cA(Z)}\log \frac{p_{\cA(Z)}(h)}{p_{\cA([Z_{-1:2},z'_1,z'_2])}(h)} \\
	=& \E_{[Z,z_1']\sim \cD^{n+1}}\E_{h\sim \cA(Z)}\left[\log \frac{p_{\cA(Z)}(h)}{p_{\cA([Z_{-1},z'_1])}(h)}\right]\\
	&+\E_{[Z,z_1',z_2']\sim \cD^{n+2}}\E_{h\sim \cA(Z)}\left[\log\frac{p_{\cA([Z_{-1},z'_1])}(h)}{p_{\cA([Z_{-1:2},z'_1,z'_2])}(h)}\right]\\
	\leq&\epsilon + \E_{[Z,z_1',z_2']\sim \cD^{n+2}}\E_{h\sim \cA(Z)}\left[\log\frac{p_{\cA([Z_{-1},z'_1])}(h)}{p_{\cA([Z_{-1:2},z'_1,z'_2])}(h)}\right].
	\end{align*}
	The technical issue is that the second term does not have the correct distribution to take expectation over.
	By the property of $\cA$ being a posterior sampling algorithm, we can rewrite the second term of the above equation into
	$$
	\E_{Z\sim \cD^n,z'_1,z'_2\sim\cD}\E_{h\sim \cA(Z)}\left[\log p(z_2,h) - \log p(z'_2,h) - \log K + \log K'\right]
	$$
	where $K$ and $K'$ are normalization constants of $\exp(\log p(z_1',h)+\sum_{i=2}^n \log p(z_i,h))$ and $\exp(\log p(z_1',h)+\log p(z_2',h)+\sum_{i=3}^n \log p(z_i,h))$ respectively. The expected log-partition functions are the same so we can replace them with normalization constants of $\exp(\sum_{i=1}^n \log p(z_i,h))$ and $\exp(\log p(z_2',h)+\sum_{i\neq2} \log p(z_i,h))$. By adding and subtracting the missing log-likelihood functions on $z_1,z_3,...,z_n$, we get 
	\begin{align*}
	&\E_{Z\sim \cD^n,z'_1,z'_2\sim\cD}\E_{h\sim \cA(Z)}\left[\log\frac{p_{\cA([Z_{-1},z'_1])}(h)}{p_{\cA([Z_{-1:2},z'_2])}(h)}\right]\\
	=& \E_{Z\sim \cD^n,z'_1,z'_2\sim\cD}\E_{h\sim \cA(Z)}\left[\log\frac{p_{\cA(Z)}(h)}{p_{\cA([Z_{-2},z'_2])}(h)}\right]\leq \epsilon
	\end{align*}
	This completes the proof for $k=2$. Apply the same argument recursively by different decompositions of , we get the results for $k=3,...,n$.
	
	The second statement follows by the same argument with all ``$\leq$'' changed into ``$=$''. \qed
\end{proof}

\begin{proof}[Proof of Lemma~\ref{lem:composition}]
	\begin{align*}
	&\E_{h_1\sim \cA(Z)} \E_{h_2\sim \cB(Z,h_1)} \log \left[\frac{p_{\cB(Z,h_1)}(h_2)}{p_{\cB(Z',h_1)}(h_2)}\frac{p_{\cA(Z)}(h_1)}{p_{\cA(Z')}(h_1)}\right]\\
	=&\E_{h_1\sim \cA(Z)} \E_{h_2\sim \cB(Z,h_1)} \log \left[\frac{p_{\cB(Z,h_1)}(h_2)}{p_{\cB(Z',h_1)}(h_2)} \right]+\E_{h_1\sim \cA(Z)} \E_{h_2\sim \cB(Z,h_1)}\log\left[\frac{p_{\cA(Z)}(h_1)}{p_{\cA(Z')}(h_1)}\right]\\
	\leq &\sup_{h_1\in \Omega_\cA}\E_{h_2\sim \cB(S,h_1)} \log \left[\frac{p_{\cB(Z,h_1)}(h_2)}{p_{\cB(Z',h_1)}(h_2)} \right] + \E_{h_1\sim \cA(Z)}\log\left[\frac{p_{\cA(Z)}(h_1)}{p_{\cA(Z')}(h_1)}\right]
	\end{align*}
	Take $\sup$ over $Z$ and $Z'$ we get the adaptive composition result for KL-Privacy. Take $\E$ over $Z$ and $z$ such that $Z' = [Z_{-1},z]$, we get the adaptive composition result for On-Average KL-Privacy.\qed
\end{proof}

\begin{proof}[Proof of Lemma~\ref{lem:maxinfo}]
	By Lemma~12 in \citet{dwork2015generalization}, $I_\infty(\cA,n) = \sup_{Z,Z'\in \cZ^n} D_\infty(\cA(Z)||\cA(Z'))$.
	\begin{align*}
	D_\infty(\cA(Z)||\cA(Z')) &= \sup_h \log \frac{p_{\cA(Z)}(h)}{p_{\cA(Z')}(h)} \\
	&=\sup_h \sum_{i=1}^n\log \frac{p_{\cA([Z_{-1:(i-1)},Z'_{1:(i-1)}])}(h)}{p_{\cA([Z_{-1:(i)},Z'_{1:(i)}])}(h)}\\
	&\geq \E_{h\sim \cA(Z)} \sum_{i=1}^n \log \frac{p_{\cA([Z_{-1:(i-1)},Z'_{1:(i-1)}])}(h)}{p_{\cA([Z_{-1:(i)},Z'_{1:(i)}])}(h)}\\
	&\geq \sum_{i=1}^n \E_{h\sim \cA(Z)} \log \frac{p_{\cA([Z_{-1:(i-1)},Z'_{1:(i-1)}])}(h)}{p_{\cA([Z_{-1:(i)},Z'_{1:(i)}])}(h)}\\
	&= \sum_{i=1}^n \E_{h\sim \cA(Z)}(\log p(z_i;h) - \log p(z'_i;h) - \log K_i  + \log K_i')
	\end{align*}
	where $K_i$ and $K_i'$ are normalization constants for distribution $p_{\cA([Z_{-1:(i-1)},Z'_{1:(i-1)}])}(h)$
	and $p_{\cA([Z_{-1:(i)},Z'_{1:(i)}])}(h)$ respectively.
	
	Take expectation over $Z$ and $Z'$ on both sides, by symmetry, the expected normalization constants are equal no matter which size $n$ subset of $[Z,Z']$ this posterior distribution $h$ is defined over. Define $Z^{(i)}:=[z_1,...,z_{i-1},z'_{i},z_{i+1},...,z_n]$. Let $K$ be the normalization constant of $\cA(Z)$ and $K^{(i)}$ be the normalization constant of $\cA(Z^{(i)})$.
	We get
	\begin{align*}
	&\E_{Z,Z'\sim \cD^n} D_\infty(\cA(Z)||\cA(Z')) \geq \E_{Z,Z'\sim \cD^n} \sum_{i=1}^n \E_{h\sim \cA(Z)}(\log p(z_i;h) - \log p(z'_i;h))\\
	=& \E_{Z,Z'\sim \cD^n} \sum_{i=1}^n \E_{h\sim \cA(Z)}\left[\sum_{j=1}^n\log p_h(z_j) - \sum_{j\neq i}\log p_h(z_j)-\log p_h(z_i') - \log K  + \log K^{(i)}\right]\\
	=&\E_{Z,Z'\sim \cD^n} \sum_{i=1}^n \E_{h\sim \cA(Z)}\left[\log \frac{p_{\cA(Z)}(h)}{p_{\cA(Z^{(i)})}(h)}\right]\\
	=&\sum_{i=1}^n \E_{Z\sim \cD^n,z'_i\sim \cD} \E_{h\sim \cA(Z)}\left[\log \frac{p_{\cA(Z)}(h)}{p_{\cA(Z^{(i)})}(h)}\right]
	\end{align*}
	Note that  
	$$\text{RHS} = n \E_{Z\sim \cD^n,z\sim\cD} \text{KL}(\cA(Z)\|\cA([Z_{-1},z])), $$
	and 
	$$\text{LHS} = \E_{Z,Z'\sim \cD^n} D_\infty(\cA(Z)\|\cA(Z')) \leq \sup_{Z,Z'\in \cZ^n} D_\infty(\cA(Z)\|\cA(Z')) = I_\infty(\cA,n).$$
	Collecting the three systems of inequalities above, we get that $\cA$ is $k/n$-On-Average-KL-Privacy as claimed. \qed
\end{proof}

\begin{proof}[Proof of Lemma~\ref{lem:mutualinfo}]
	Denote	$p(\cA(Z))  = p(h|Z)$. $p(h,Z) = p(h|Z) p(Z)$. The marginal distribution of $h$ is therefore 
	$p(h)=\int_Z p(h,Z) dZ = \E_{Z} p(\cA(Z)) $. By definition,
	\begin{align*}
	& I(\cA(Z);Z) = \E_{Z}\E_{h|Z} \log \frac{p(h|Z) p(Z)}{p(h)p(Z)} \\
	=& \E_{Z}\E_{h|Z} \log p(h|Z) - \E_{Z}\E_{h|Z}\log \E_{Z'} p(h|Z')\\
	=& \E_{Z}\E_{h|Z} \log p(h|Z) -  \E_{Z,Z'}\E_{h|Z} \log p(h|Z')\\
	& +\E_{Z,Z'}\E_{h|Z} \log p(h|Z')- \E_{Z}\E_{h|Z}\log \E_{Z'} p(h|Z')\\
	=&\E_{Z,Z'}\E_{h|Z} \log \frac{p(h|Z)}{p(h|Z')}+\E_{Z,Z'}\E_{h|Z} \log p(h|Z')- \E_{Z}\E_{h|Z}\log \E_{Z'} p(h|Z')\\
	=& D_{\mathrm{KL}}(\cA(Z),\cA(Z')) + \E_{Z}\E_{h|Z}\left[ \E_{Z'}\log p(\cA(Z')) - \log \E_{Z'}p(\cA(Z')) \right]\\
	\leq& D_{\mathrm{KL}}(\cA(Z),\cA(Z')).
	\end{align*}
	The last line follows from Jensen's inequality. \qed
\end{proof}

\bibliographystyle{splncsnat}
\bibliography{personal-privacy}

\end{document}